\documentclass{article}
\usepackage{spconf}
\usepackage{bm}
\usepackage{amsthm}

\newcommand{\tensor}[1]{\bm {\mathcal{#1}}}

\usepackage{todonotes}

\usepackage{multicol, blindtext}
\usepackage{amsmath,amssymb, graphicx}
\usepackage[ruled,linesnumbered]{algorithm2e}
\usepackage{algorithmic}
\usepackage{paralist} 
\usepackage{color}
\usepackage{multirow}
\usepackage{rotating}
\usepackage{hyperref}
\usepackage[mathscr]{eucal} 
\usepackage{amsbsy} 
\usepackage{subfigure}
\usepackage{booktabs}
\usepackage{xcolor}
\usepackage{tikz}
\usetikzlibrary{snakes}
\usepackage{multirow}
\usepackage{epstopdf}
\usepackage{balance}
\usepackage{tablefootnote}
\usepackage{empheq} 
\usepackage{moresize}
\usepackage{enumitem}
\setlist{nolistsep}

\usepackage[font=small,labelfont=bf,skip=0pt]{caption}

\DeclareCaptionType{copyrightbox}
\usepackage{url}

\newcounter{ALC@tempcntr}
\newcommand{\hide}[1]{}

\newcommand{\ben}{\begin{enumerate*}}
\newcommand{\een}{\end{enumerate*}}
\newcommand{\bit}{\begin{itemize*}}
\newcommand{\eit}{\end{itemize*}}

\def\x{{\mathbf x}}

\DeclareMathOperator*{\argmin}{arg\,min}

\begin{document}

\name{Georgios Tsitsikas, Evangelos E. Papalexakis}
  \address{ Dept. of Computer Science and Engineering\\
  University of California, Riverside}
\title{THE CORE CONSISTENCY OF A  COMPRESSED TENSOR}
\maketitle
\begin{abstract}
Tensor decomposition on big data has attracted significant attention recently. Among the most popular methods is a class of algorithms that leverages compression in order to reduce the size of the tensor and potentially parallelize computations. A fundamental requirement for such methods to work properly is that the low-rank tensor structure is retained upon compression. In lieu of efficient and realistic means of computing and studying the effects of compression on the low rank of a tensor, we study the effects of compression on the core consistency; a widely used heuristic that has been used as a proxy for estimating that low rank. We provide theoretical analysis, where we identify sufficient conditions for the compression such that the core consistency is preserved, and we conduct extensive experiments that validate our analysis. Further, we explore popular compression schemes and how they affect the core consistency.
\end{abstract}

\begin{keywords}
Tensor decomposition, tensor rank, core consistency, CORCONDIA, compression
\end{keywords}

\section{Introduction}
\label{sec:intro}
In recent years, there has been a tremendous increase in the amount of data being available in many application areas of interest \cite{PapFalSid:TIST2016}. These data are many times multi-aspect and, therefore, they are very elegantly described by multidimensional arrays where each index of the array corresponds to a specific aspect of the data. 

Tensors, which are very closely tied to multidimensional arrays, have proved to be a very useful framework for analyzing and extracting structure from these data, which is usually achieved by employing tensor decompositions. A plethora of such decompositions has been suggested \cite{PARAFAC,tucker3,sidiropoulos2016tensor}, many of which have also seen significant algorithmic advances especially when dealing with big data \cite{kolda2008scalable,papalexakis2012parcube,kang2012gigatensor}. 

A particular scheme of interest, though, is the one where the tensor data are randomly compressed before being decomposed \cite{sidiropoulosparallel,sidiropoulos2014parallel,yang2018parasketch}. However, even though the analysis in those papers is sound, all results are predicated on the fact that the rank is preserved during this compression/sketching step. In practical cases, we have observed this to be true, but, to the best of our knowledge, there is no analysis of this phenomenon. Thus, in this paper we set out to investigate the effects of such compression on the rank of a tensor, assuming that it has low-rank trilinear structure.

To this end, and realizing that tensor rank calculation is an NP-hard problem \cite{haastad1990tensor}, we will focus on the impact of compression on the Core Consistency Diagnostic \cite{bro1998multi,bro2003new}, which is the most widely used heuristic for judging the trilinearity of a tensor and identifying the rank \cite{papalexakis2016autoten,kamstrup2013core,holbrook2006characterizing}. Specifically, our contributions are:
\begin{itemize}
    \item {\bf Theoretical analysis}: We provide a sufficient condition for the compression matrices, under which the compressed tensor preserves the Core Consistency.
    \item {\bf Extensive experimental evaluation}: We thoroughly evaluate our theoretical results using real sugar data that have been chemically verified for their trilinearity \cite{bro2003new}. To achieve this, various experiments with different schemes of compression were carried out, and their efficiency in retaining the Core Consistency of a tensor with low-rank trilinear structure is discussed. 
\end{itemize}


\section{Problem Formulation}
\label{sec:problem}

\subsection{Notation \& Definitions}
An $N$-mode tensor $\tensor{X}$ can be defined as an element of the tensor product of $N$ vector spaces, and by choosing bases for these spaces, it can be described as a multidimensional array of numbers. Specifically, for the tensor product of $N$ vector spaces of dimensions $I_1$, $I_2$, $\cdots$, $I_N$, by choosing bases in $\mathbb{R}^{I_1}$, $\mathbb{R}^{I_2}$, $\cdots$, $ \mathbb{R}^{I_N}$, respectively, $\tensor{X}$ can be expressed as an element of $\mathbb{R}^{I_1\times I_2 \times \cdots \times I_N}$. 

\textbf{Mode-n Fiber:} A mode-$n$ fiber of $\tensor{X}$ is the column vector $\tensor{X}( \cdots, i_{n-1}, :, \cdots)$, whose elements are the elements of $\tensor{X}$ for $i_n = 1,\cdots, I_n$ while the rest of the indices are fixed. 

\textbf{n-Mode Product:} The $n$-mode product of $\tensor{X}$ with a matrix $\mathbf{Z} \in \mathbb{R}^{K\times I_n}$, is denoted by $\tensor{X}\times_n \mathbf{Z}$, and results in a tensor whose $n$-mode fibers are the $n$-mode fibers of $\tensor{X}$ multiplied by $\mathbf{Z}$. In other words, 
\begin{equation*}
    (\tensor{X}\times_n \mathbf{Z})( \cdots, i_{n-1}, :, \cdots) = \mathbf{Z} \cdot \tensor{X}(\cdots, i_{n-1}, :, \cdots)
\end{equation*}
Note that $(\tensor{X}\times_n \mathbf{Z})  \in \mathbb{R}^{I_1 \times \cdots \times I_n-1 \times K \times \cdots \times I_N} $.

\textbf{Frobenius Norm:} The Frobenius Norm $||\cdot||$ of a tensor $\tensor{X}$ is defined as 
\begin{equation*}
    ||\tensor{X}||=\sqrt{\sum_{i_1=1}^{I_1}\sum_{i_2=1}^{I_2}\cdots\sum_{i_N=1}^{I_N}\tensor{X}(i_1,i_2,\cdots,i_N)^2}
\end{equation*}
\subsection{Tensor Decompositions}
\label{sec:tens_dec}
Many times we are interested in expressing $\tensor{X}$ in a decomposed form, since this can be instrumental in different data analytics scenarios \cite{PapFalSid:TIST2016,sidiropoulos2016tensor}.


Particularly, we can decompose $\tensor{X}$ as the sum of rank-one tensors. In this paper, we will only consider $3$-mode tensors, and, therefore, these rank-one tensors will be the outer product of three vectors, $\mathbf{a}_p \in \mathbb{R}^I$ with $p=1,\dots,P$, $\mathbf{b}_q \in \mathbb{R}^J$ with $q=1,\dots,Q$, and $\mathbf{c}_r \in \mathbb{R}^K$ with $r=1,\dots,R$. These vectors can be grouped as the columns of three factor matrices $\mathbf{A}$, $\mathbf{B}$ and $\mathbf{C}$, respectively, so that $\mathbf{A} \in \mathbb{R}^{I\times P}$, $\mathbf{B} \in \mathbb{R}^{J\times Q}$ and $\mathbf{C} \in \mathbb{R}^{K\times R}$.

One such decomposition is PARAFAC \cite{PARAFAC,bro1997parafac}, which allows us to express $\tensor{X}$ as
    \begin{equation}
        \tensor{X}  = \sum_{r=1}^R \mathbf{a}_r \circ     \mathbf{b}_r \circ  \mathbf{c}_r 
        \label{parafac_out}
    \end{equation}  
or equivalently
    \begin{equation}
        \tensor{X} = \tensor{I} \times_1 \mathbf{A} \times_2 \mathbf{B} \times_3 \mathbf{C} 
        \label{parafac}
    \end{equation}
where $R$ is the number of components, and $\tensor{I}(i,j,k)$ is 1 for $i=j=k$ and zero everywhere else.

TUCKER3 is another useful decomposition \cite{tucker3}, which generalizes PARAFAC, and allows us to express $\tensor{X}$ as
\begin{equation}
        \tensor{X}  = \sum_{p=1}^P\sum_{q=1}^Q\sum_{r=1}^R \tensor{G}(p,q,r) \cdot \mathbf{a}_p \circ     \mathbf{b}_q \circ  \mathbf{c}_r 
         \label{tucker_out}
\end{equation} 
or equivalently
    \begin{equation}
        \tensor{X} = \tensor{G} \times_1 \mathbf{A} \times_2 \mathbf{B} \times_3 \mathbf{C}
        \label{tucker}
    \end{equation}
where $\tensor{G}$ is called the TUCKER3 core.

Additionally, since in our work we assume low-rank structure,  we will only consider tall PARAFAC and tall orthonormal TUCKER3 factor matrices. The orthonormality assumption might seem restrictive at first glance, but observe that for non-orthonormal tall TUCKER3 factor matrices we can employ their reduced QR factorization to get
    \begin{align*}
        \tensor{X} &= \tensor{G} \times_1 \mathbf{Q_A R_A} \times_2 \mathbf{Q_B R_B} \times_3 \mathbf{Q_C R_C}\\
        &= (\tensor{G} \times_1 \mathbf{R_A} \times_2 \mathbf{R_B} \times_3 \mathbf{R_C})\times_1 \mathbf{Q_A} \times_2 \mathbf{Q_B} \times_3 \mathbf{Q_C}\\
        &= \tensor{\widetilde{G}} \times_1 \mathbf{Q_A} \times_2 \mathbf{Q_B} \times_3 \mathbf{Q_C}
    \end{align*}

Finally, it should be mentioned that for different values of $R$ in \eqref{parafac_out} we get different decompositions, and the same is true for different values of $P$, $Q$ and $R$ in \eqref{tucker_out}. Keep in mind, however, that for some values an exact decomposition may not exist at all. 
\subsection{Core Consistency Diagnostic}
The Core Consistency Diagnostic (CORCONDIA) \cite{bro1998multi,bro2003new} is defined as 
\begin{equation*}
    \left(1-\frac{||\tensor{I}-\tensor{G}||^2}{||\tensor{I}||^2}\right)\cdot 100
\end{equation*}
where
\begin{equation}
    \tensor{G} = \tensor{X}\times_1 \mathbf{A}^+\times_2 \mathbf{B}^+\times_3 \mathbf{C}^+
    \label{corc_core}
\end{equation}
and  $\mathbf{A}^+$, $\mathbf{B}^+$ and $\mathbf{C}^+$ are the Moore-Penrose inverses of the PARAFAC factor matrices of $\tensor{X}$.

Expression \eqref{corc_core} is obtained as the minimum norm solution of the least squares problem
\begin{equation*}
       \argmin_{\tensor{G}}||\tensor{X}-\tensor{G}\times_1 \mathbf{A}\times_2 \mathbf{B}\times_3 \mathbf{C}||
\end{equation*}
Note that PARAFAC can be expressed as the solution of the least squares problem
\begin{equation*}
       \argmin_{\mathbf{A}, \mathbf{B}, \mathbf{C}}||\tensor{X}-\tensor{I}\times_1 \mathbf{A}\times_2 \mathbf{B}\times_3 \mathbf{C}||
\end{equation*}
Hence, we can see that CORCONDIA essentially attempts to quantify how well a PARAFAC decomposition describes $\tensor{X}$, by comparing to how well $\tensor{X}$ can be described when interactions between all the columns of $\mathbf{A}$, $\mathbf{B}$ and $\mathbf{C}$ are allowed. 

Specifically, when these interactions do not improve the model significantly, then we can interpret it as a sign that the PARAFAC model is appropriate. Further, $\tensor{G}$ will have its dominant elements on the diagonal, and, thus, CORCONDIA will have a value close to 100.
On the other hand, when the interactions produce a substantially better model, then the PARAFAC model is probably not appropriate. In fact, $\tensor{G}$ will have many off-diagonal elements, which will result in a close to zero, or even negative CORCONDIA.

This means that we can evaluate CORCONDIA on a range of PARAFAC decompositions with different number of components, $R$, and select the one with the largest number of components that also retains a reasonably high CORCONDIA value. By using this method, we hope to discover a PARAFAC model that best describes potential trilinear variation in our data.

\section{Proposed Analysis}
\label{sec:method}

Although CORCONDIA is a very useful diagnostic for discovering trilinear variation in data, there can be times where it becomes very computationally and memory intensive in practice. Specifically, if we consider very large tensors, not only the PARAFAC factor matrices and their pseudoinverses can take prohibitive amounts of time to calculate, but in cases where the whole tensor cannot fit into the main memory, performance can deteriorate substantially.

For this reason, it would be useful to have in our disposal a tool that allows us to extract and study a much smaller version of the tensor that, despite its small size, retains most of the systematic variation. However, since finding such a compressed tensor can also be computationally expensive, we need to settle for a trade-off between how fast and how accurately it can be generated. 

To this end, we propose to study the statistics of multiple randomly compressed tensors, $\tensor{X}'$, by employing $n$-mode products of the uncompressed tensor, $\tensor{X}$, with matrices $\mathbf{U} \in \mathbb{R}^{L\times I}$, $\mathbf{V} \in \mathbb{R}^{M\times J}$ and $\mathbf{W} \in \mathbb{R}^{N\times K}$ having orthonormal rows, so that $\tensor{X}'$ can be expressed as 
\begin{equation*}
    \tensor{X}' = \tensor{X} \times_1 \mathbf{U} \times_2 \mathbf{V} \times_3 \mathbf{W} 
\end{equation*}
where $\tensor{X}' \in \mathbb{R}^{L \times M \times N}$ with $L<I$, $M<J$ and $N<K$.

The main reason for selecting this kind of compression is on one hand because of its simplicity, but at the same time because it allows us to derive elegant and useful theoretical results, as shown in the following claims.

\newtheorem{claim}{Claim}

\begin{claim}\label{claim_tucker}
When $\tensor{X}$ has an exact PARAFAC decomposition and its mode-1, mode-2 and mode-3 fibers belong in the rowspace of $\mathbf{U}$, $\mathbf{V}$ and $\mathbf{W}$, respectively, then CORCONDIA is preserved.
\end{claim}

\begin{proof}
    Since $\tensor{X}$ has a PARAFAC decomposition, we get
    \begin{equation*}
        \tensor{X}' = \tensor{I} \times_1 \mathbf{UA} \times_2 \mathbf{VB} \times_3 \mathbf{WC}
    \end{equation*}
    and, therefore, the PARAFAC decomposition of $\tensor{X}'$ is given by $\mathbf{A}'=\mathbf{UA}$, $\mathbf{B}'=\mathbf{VB}$ and $\mathbf{C}'=\mathbf{WC}$. As a result, \eqref{corc_core} gives
    \begin{align*}
            \tensor{G}' &= \tensor{X}'\times_1 (\mathbf{UA})^+\times_2 (\mathbf{VB})^+\times_3 (\mathbf{WC})^+\\
                &= \tensor{X}\times_1 (\mathbf{UA})^+\mathbf{U}\times_2 (\mathbf{VB})^+\mathbf{V}\times_3 (\mathbf{WC})^+\mathbf{W}\\
                &= \tensor{X}\times_1 \mathbf{A}^+\mathbf{U}^T\mathbf{U}\times_2 \mathbf{B}^+\mathbf{V}^T\mathbf{V}\times_3 \mathbf{C}^+\mathbf{W}^T\mathbf{W}\\
                &= \tensor{X}_{pr}\times_1\mathbf{A}^+\times_2\mathbf{B}^+\times_3\mathbf{C}^+
    \end{align*}
    where 
   $     \tensor{X}_{pr} = \tensor{X}\times_1 \mathbf{U}^T\mathbf{U}\times_2 \mathbf{V}^T\mathbf{V}\times_3 \mathbf{W}^T\mathbf{W},$
    which can be seen as the projection of $\tensor{X}$ onto the rowspaces of $\mathbf{U}$, $\mathbf{V}$ and $\mathbf{W}$. Finally, since the fibers of $\tensor{X}$ belong in the rowspaces of these matrices, it will hold that $\tensor{X}_{pr}$ = $\tensor{X}$, which in turn gives
    \begin{equation*}
        \tensor{G}' = \tensor{X} \times_1\mathbf{A}^+\times_2\mathbf{B}^+\times_3\mathbf{C}^+ = \tensor{G}
    \end{equation*}
    Therefore, CORCONDIA is preserved.
\end{proof}
\begin{claim}
When $\tensor{X}$ has an exact PARAFAC decomposition and we compress using the transpose of its TUCKER3 factor matrices, then CORCONDIA is preserved.
\end{claim}
\begin{proof}
First, notice that when an exact PARAFAC decomposition exists, then an exact TUCKER3 decomposition will also exist, which can be derived from the PARAFAC decomposition by employing the reduced QR factorization of its factor matrices as shown in \autoref{sec:tens_dec}. Now from \eqref{tucker} we get
    \begin{align*}
        \tensor{X} \times_1 \mathbf{A}^T \times_2 \mathbf{B}^T \times_3 \mathbf{C}^T &= \tensor{G} \implies \\
        \tensor{X} \times_1 \mathbf{A}\mathbf{A}^T \times_2 \mathbf{B}\mathbf{B}^T \times_3 \mathbf{C}\mathbf{C}^T &= \tensor{G}\times_1 \mathbf{A} \times_2 \mathbf{B} \times_3 \mathbf{C} \implies \\
        \tensor{X}_{pr} &= \tensor{X}
    \end{align*}
    and, thus, we conclude that all mode-1, mode-2 and mode-3 fibers belong in the rowspace of $\mathbf{A}^T$, $\mathbf{B}^T$ and $\mathbf{C}^T$, respectively. At this point, all conditions in Claim \autoref{claim_tucker} are satisfied, which means that CORCONDIA is preserved. 
\end{proof}

We should mention that such a compression scheme has also been studied in \cite{bro1998improving} 
in the context of speeding up the calculation of PARAFAC decompositions. That said, even though that work can provide further insight into why this compression scheme is sensible, our approach differs in that instead of looking for optimal compression matrices, it takes the more agnostic path of multiple random compressions. 

\section{Experimental Evaluation}
\label{sec:experiments}
In this section, we present experimental results on the behavior of CORCONDIA on compressed real tensor data. Specifically, we are studying sugar data of size 268$ \times $44$ \times $7 that are known to have trilinear structure \cite{bro2003new}. All experiments were run on a system with an Intel(R) Core(TM) i5-7300HQ and 8GB of RAM. Matlab along with Tensor Toolbox \cite{tensortoolbox} was used throughout the whole process, except for the calculation of all PARAFAC and TUCKER3 decompositions for which N-way Toolbox \cite{nwaytoolbox} was utilized. 

We have experimented with the following three types of compression:
\begin{itemize}
    \item \textbf{Gaussian} - the tensor is multiplied modewise with matrices whose elements are independent identically distributed random variables following the standard normal distribution.
    \item \textbf{Orthonormal} - the tensor is multiplied modewise with random matrices with orthonormal rows. These matrices are obtained from the reduced QR decomposition of the transpose of a Gaussian compression matrix.
    \item  \textbf{Tucker} - the tensor is multiplied modewise with the transpose of its  TUCKER3 factor matrices, which in fact results in a compressed tensor identical to the TUCKER3 core.
\end{itemize}

\begin{figure*}[ht]
\centering
$\begin{array}{cc}
\includegraphics[width=\columnwidth]{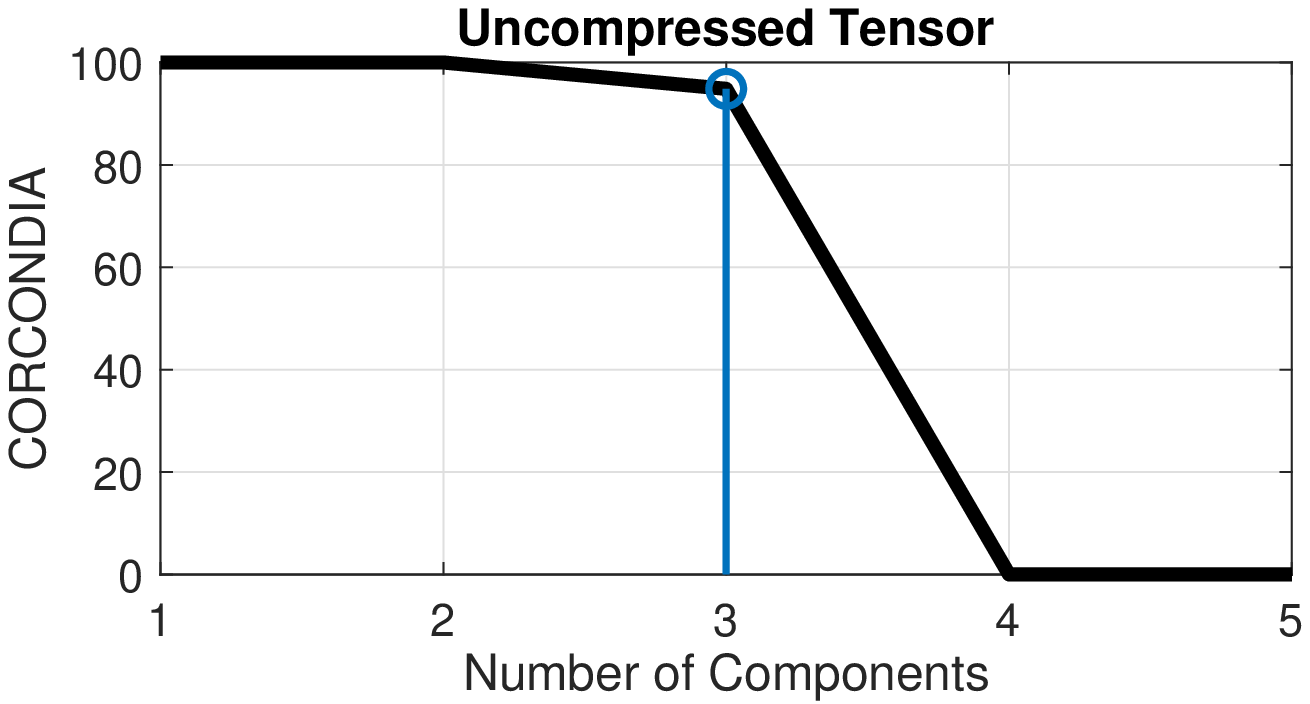} &
\includegraphics[width=\columnwidth]{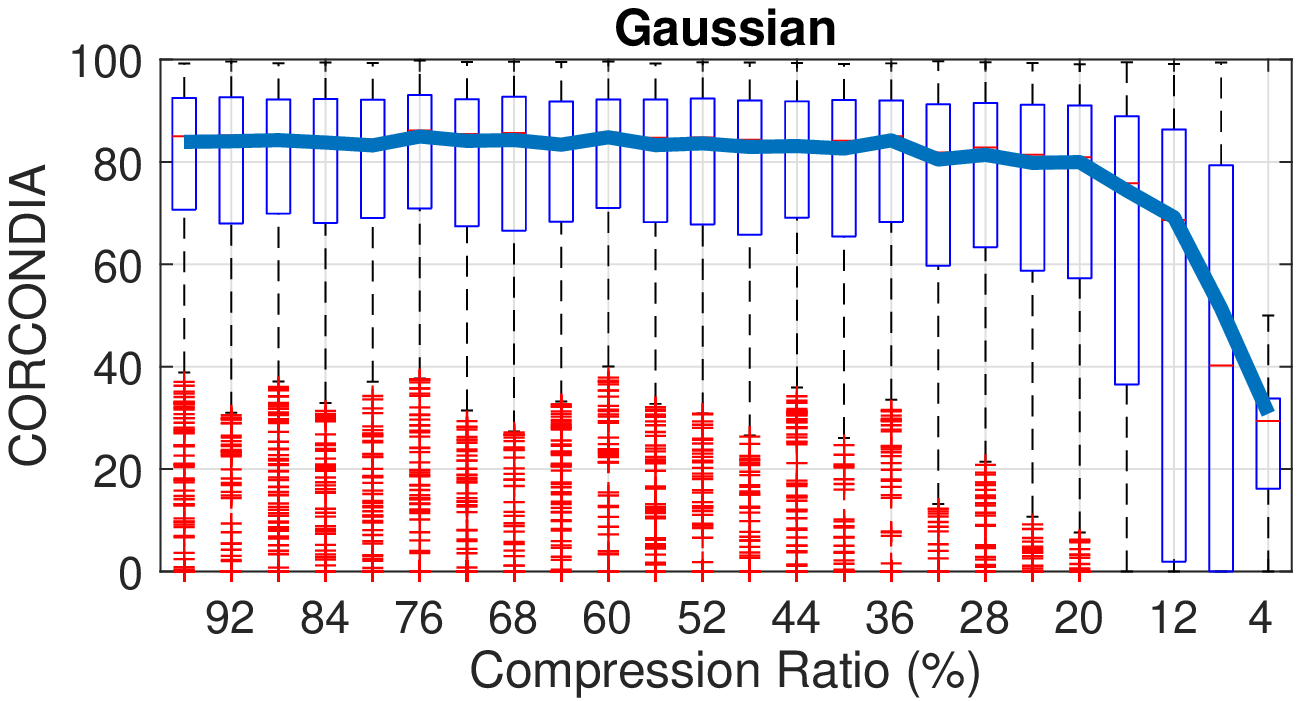}\\
\includegraphics[width=\columnwidth]{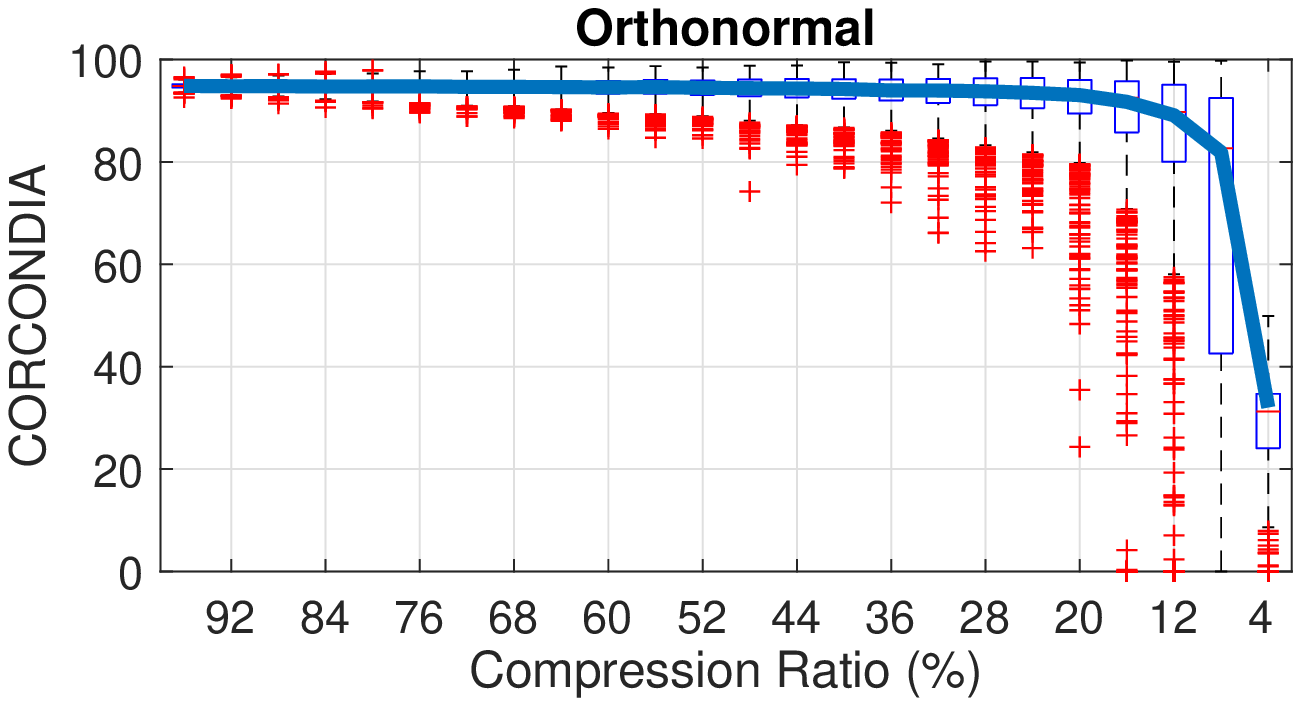} &
\includegraphics[width=\columnwidth]{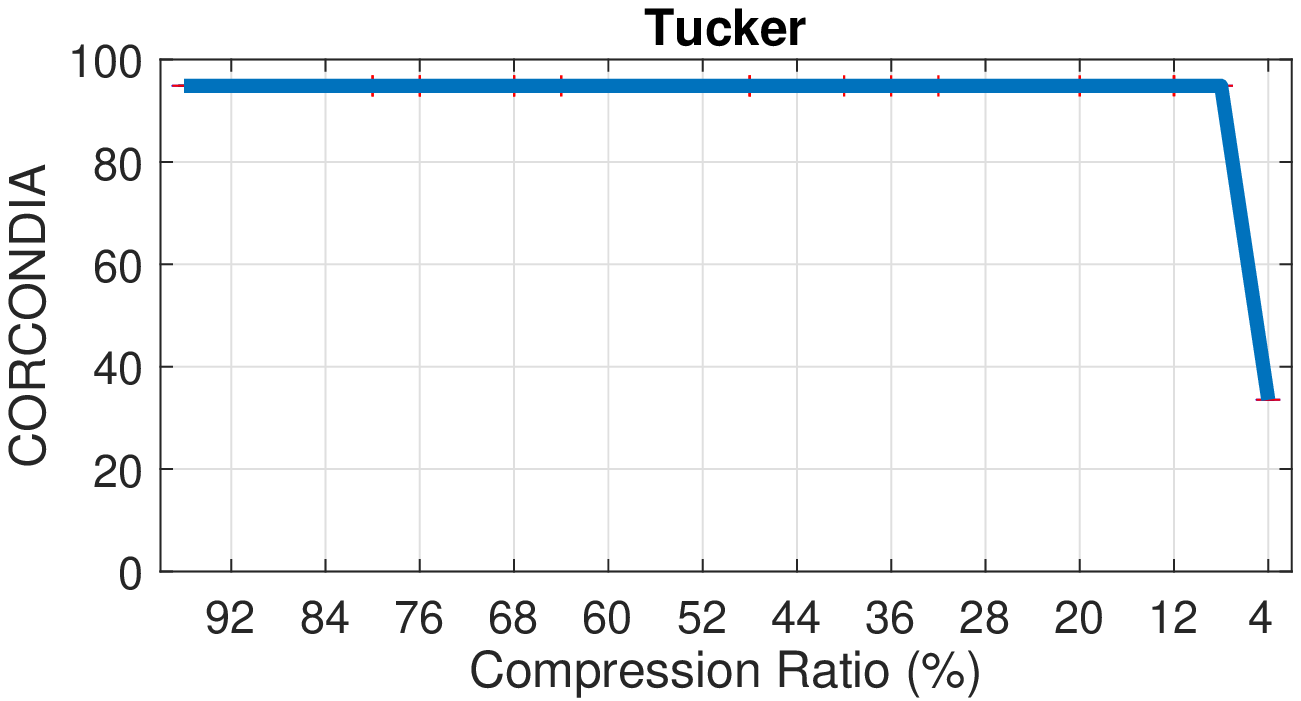}
\end{array}
$
\caption{Comparison of Gaussian, Orthonormal and Tucker compressions for three PARAFAC components.}
\label{fig:all}
\end{figure*}

In \autoref{fig:all} we present the results for all three types of compression. Specifically, the upper left plot shows the CORCONDIA values of the uncompressed tensor for multiple PARAFAC decompositions with different number of components. The vertical line indicates the fixed number of components at which the compressed tensors are tested, while the circle shows the value of CORCONDIA that they are expected to retain for various compression types and ratios. 

The rest of the plots show CORCONDIA for each of the aforementioned compression types, where 1000 samples were used for Gaussian and Orthonormal compression, and 10 samples for Tucker compression. They contain the boxplots for each compression ratio along with the corresponding outliers, which are denoted by `+'. All calculations were done after negative CORCONDIA values were set equal to zero. Further, the sample means for all compression ratios  are represented by the thick lines, and they were calculated after the outliers were first properly smoothed. Finally, the compression ratio is the percentage by which we reduce the size of the first and the second mode; we do not compress in the third mode since it is already really small. For instance, for a compression ratio of 50\%, we get a compressed tensor of size 134$ \times $22$ \times $7. We reserve more detailed compression ratio schemes for the extended version of this work.


Note that CORCONDIA seems to not be preserved only when it starts falling from 100 to 0. This occurs at the PARAFAC model with 3 components, and this is why only this case is discussed in this paper. For 1, 2, 4 and 5 components, it is almost perfectly preserved for all compression types and for all compression ratios up to even 4\%. 

Examining the plots in \autoref{fig:all}, it is clear that Tucker compression achieves the best performance by managing to perfectly preserve CORCONDIA up to 8\% compression. Orthonormal compression comes second by generally preserving CORCONDIA up to about 20\% compression, although in the mean it retains the value high enough up to 8\% compression. 
In fact, we can expect very similar performance even with very few samples since the variance is very small. On the other hand, Gaussian compression has clearly a much worse performance, not only because it rarely retains the original CORCONDIA value, but also because for multiple compressions the variance is too large. That said, it can retain in the mean a high enough CORCONDIA up to 20\% compression.

At this point, one might feel tempted to consider Tucker compression as the undisputed winner. However, we should not forget that this compression type requires the calculation of the TUCKER3 decomposition of the uncompressed tensor which can actually become very  computationally expensive. On the other hand, the Gaussian and Orthonormal methods can perform the compression much faster since we can generate the compression matrices easier.

\section{Conclusion}
\label{sec:conclusions}
We study the effects of various popular and practical tensor compression schemes in the CORCONDIA of a tensor. We provide theoretical insights on conditions that satisfy perfect retention of CORCONDIA upon compression, along with experimental results that verify our analysis. Further, we evaluate the effect of different random compression schemes to CORCONDIA. Experimental results on real data indicate that it is possible to perform significantly tight compressions without having a serious impact on the value of CORCONDIA. Therefore, this method can be used as a tool to mitigate the consequences of the high time complexity of the calculations that CORCONDIA has to perform, especially on big tensor data.

\section{Acknowledgements}
{\scriptsize
We are grateful to Rasmus Bro and Anne Bech Risum for their valuable feedback, and for providing us with real chemical data to evaluate our compression methods.  Research was supported by the National Science Foundation CDS\&E Grant no. OAC-1808591 and by the Department of the Navy, Naval Engineering Education Consortium under award no. N00174-17-1-0005. Any opinions, findings, and conclusions or recommendations expressed in this material are those of the author(s) and do not necessarily reflect the views of the funding parties.
}

{\scriptsize
}

\balance
\bibliographystyle{IEEEtran}
\bibliography{BIB/vagelis_refs.bib}

\end{document}